\documentclass[conference]{IEEEtran}
\IEEEoverridecommandlockouts
\usepackage{lineno,hyperref}
\modulolinenumbers[10]

\usepackage{amsmath,mathtools,dsfont,amssymb}
\usepackage{todonotes}
\usepackage{bm}
\usepackage{xcolor}
\usepackage{amsthm}
\makeatletter
\newcommand{\leqnomode}{\tagsleft@true\let\veqno\@@leqno}
\newcommand{\reqnomode}{\tagsleft@false\let\veqno\@@eqno}
\makeatother

\newtheorem{theorem}{Theorem}
\newtheorem{lemma}[theorem]{Lemma}

\newcommand{\comm}[1]{}
\newcommand{\M}{\mathcal{M}}

\numberwithin{equation}{section}

\newcommand{\ignore}[1]{ }

\def\BibTeX{{\rm B\kern-.05em{\sc i\kern-.025em b}\kern-.08em
    T\kern-.1667em\lower.7ex\hbox{E}\kern-.125emX}}
\begin{document}

\title{A Concentration Bound for Distributed Stochastic Approximation\\
{\footnotesize \textsuperscript{}}
\thanks{VB was supported by the S.\ S.\ Bhatnagar Fellowship from Council of Scientific and Industrial Research,  Government of India. The authors thank Siddharth Chandak for a careful scrutiny of the manuscript.}
}

\author{\IEEEauthorblockN{Harsh Dolhare and Vivek Borkar}
\IEEEauthorblockA{\textit{Dept. of Electrical Engineering} \\
\textit{Indian Institute of Technology Bombay}\\
Mumbai, India \\
}
}

\maketitle

\begin{abstract}
We revisit the classical model of Tsitsiklis, Bertsekas and Athans \cite{TBA} for distributed stochastic approximation with consensus. The main result is an analysis of this scheme using the `ODE' (for `Ordinary Differential Equations') approach to stochastic approximation, leading to a high probability bound for the tracking error between suitably interpolated iterates and the limiting differential equation. Several future directions will also be highlighted.
\end{abstract}


\section{Introduction}

In a landmark work, Tsitsiklis, Bertsekas and Athans \cite{TBA} laid down a paradigm for distributed incremental recursions with consensus, which has been a template for much subsequent work in distributed algorithms, particularly for distributed optimization (see, e.g., \cite{Nedic}). The basic idea is to use a classical `gossip' scheme for averaging iterates across processors/agents, with an additive perturbation given by a processor-specific stochastic approximation. It is well known that under reasonable conditions, a stand-alone stochastic approximation iteration asymptotically tracks a limiting ordinary differential equation (ODE) (see, e.g., \cite{Borkar}). In the present scenario, the averaging is on a faster time scale dictated by the iteration count whereas the stochastic approximation is on a slower time scale dictated by the stepsizes (or `learning parameters') chosen. The net effect then is to confine the limiting dynamics of the latter to the invariant subspace of the former, i.e., the one dimensional space of constant vectors. What this translates into is that the iterates of all the processors asymptotically track  common trajectory segments of an `averaged' ODE with probability one, which is tantamount to `consensus' in a generalized sense. This interpretation (see, e.g., section 8.4, \cite{Borkar}, also, \cite{Mathkar}) differs from the original approach of \cite{TBA}, but gives a different and useful perspective.

Operating within this framework, our objective here is to derive an estimate for the `trapping probability', i.e., the probability that the iterates converge to a specific asymptotically stable attractor if they are in its domain of attraction at some time. This can be combined with finite time estimates, if available, to give an `all time' bound from time zero. The key step is to bound the departure of suitably interpolated iterates from the trajectory of the averaged ODE (Lemma \ref{bound-lemma} below). The remainder of the proof then closely mimics that for the centralized stochastic approximation (section  3.1, \cite{Borkar}).

\section{Algorithm}

Following \cite{TBA}, a standard distributed stochastic approximation scheme with $M$ agents in $\Re^{d}$ can be described as follows. Consider a connected directed graph $\mathcal{G} = \{V,\mathcal{E}\}$ such that $|V| = M$. Each node of this graph represents one agent. We assume that there is a directed path from each node of the graph to every other node, i.e.\  $\mathcal{G}$ is irreducible. 
We are given a stochastic matrix $P = [p_{i,j}] \in \Re^{M\times M}$ compatible with $\mathcal{G}$ (i.e.\ the edge $(i,j) \in \mathcal{E} \Longleftrightarrow p_{i,j} > 0$), with $p_{i,j} :=$  the weight assigned to the edge $i\rightarrow j$. At each time $n$, node $i$ updates its $d$-dimensional iterate $x^i(n)$ as 
\begin{equation}\label{single-eq}
    x^i(n+1) = \sum_{j=1}^{M}p_{i,j}x^{j}(n) + a(n)\big(h^{i}(x^{i}(n))+\M^{i}(n+1)\big).
\end{equation}
Here $a(n) \geq 0$ is the stepsize, $h^{i}:\Re^d\rightarrow\Re^d$, and $\{\M^i(n+1)\}$ is a $\Re^{d}$-valued `martingale difference' noise defined below.  
Define $X(n),\widetilde{\M}(n+1), h(X(n)) \in \mathbf{R}^{M\times d}$ by: the $i$th row of $X(n)$ ($\widetilde{\M}(n+1), h(X(n))$, resp.) is $x^i(n)$ ($\M^i(n+1), h^i(x^i(n))$, resp.). Then (\ref{single-eq}) can be written in matrix form as
\begin{equation}\label{matrix-eq}
    X(n+1) = PX(n) + a(n)\left(h(X(n))+\widetilde{\M}(n+1)\right)
\end{equation}
Let $\pi$ be the unique stationary distribution for $P$, $\mathbf{1}$ the vector $[1,1,\cdots,1]^T \in \Re^M$, and and $\Pi = \mathbf{1}\pi^T$. Define $Q:= P-\Pi$. Then all eigenvalues of $Q$ have magnitude strictly less than 1. It is well known \cite{Kailath} that this is so if and only if for every positive definite matrix $R\in \mathbb{R}^{M\times M}$, there exits a unique positive definite matrix $H \in \mathbb{R}^{M\times M}$ satisfying the discrete Lyapunov equation
\begin{equation}\label{Lia}
    Q^THQ - H = -R.
\end{equation}
We take $R=I$. Let $\|\cdot\|_H$  denote the norm $\sqrt{x^THx}$ and also the corresponding induced matrix norm. 

\begin{lemma}\label{contraction}  $Q$ is an $\| \cdot \|_H$-norm contraction, i.e.\ there exists an $\alpha\in(0,1)$ such that $\|Qx - Qy\|_H \leq \alpha\|x-y\|_H$ $\forall \ x,y$. \end{lemma}

\begin{proof} For any $x \in \mathbb{R}^{d}$, 
\begin{align}\label{contraction-factor}
&x^TQ^THQx = x^THx - x^Tx \nonumber \\
\implies &\|Qx\|^2_H =\|x\|^2_H - x^Tx.
\end{align}
 Let $(\lambda_i > 0, v_i)$ denote the eigenvalue - normalized eigenvector pairs of $H$. Then since $H > I$ in the usual order for positive definite matrices (as is obvious from (\ref{Lia}) with $R = I$), we have $\lambda_i \geq 1 \ \forall i$. Hence
\begin{eqnarray*}
\lefteqn{\alpha \coloneqq \max_{\|x\| \neq 0}\frac{\|Qx\|_H}{\|x\|_H} = \max_{\|x\| \neq 0}\frac{\sqrt{x^TQ^THQx}}{\sqrt{x^THx}}} \\
&=&  \max_{\|x\| \neq 0}\frac{\sqrt{x^T(H - I)x}}{\sqrt{x^THx}} \\
&=& \max_{\|x\| \neq 0}\frac{\sqrt{\sum_i(\lambda_i - 1)\langle x, v_i\rangle^2}}{\sqrt{\sum_i\lambda_i\langle x, v_i\rangle^2}} \ < \ 1.
\end{eqnarray*}
\end{proof}

We now state our key assumptions.
\begin{itemize}
    \item $h^{i}$ (hence $h$) are Lipshitz, i.e.\  there exists a constant $L>0$ such that
    \begin{equation*}
        \|h(X)-h(Y)\|_H \leq L\|X-Y\|_H,
    \end{equation*}
implying in particular that for a suitable $K_1>0$,
    \begin{equation*}
        \|h(X)\|_2 \leq K_1(1+\|X\|_2).
    \end{equation*}
    \item $\M^{i}(n+1)$ is a $\Re^{d}$-valued martingale  difference sequence  with respect to the increasing  $\sigma$-fields $\mathcal{F}_n:= \sigma(X(k),\M(k), k\leq n)$ (i.e., $E\left[\M^i(n+1)|\mathcal{F}_n\right] = \theta :=$ the zero vector),
 satisfying
    \begin{equation*}
        \|\widetilde{\M}(n+1)\|_H \leq K_2(1+\|X(n)\|).
    \end{equation*}
Furthermore, we assume that there exist constants $\kappa, C > 0$ such that
\begin{equation}\label{expbound}
\sup_{i,j,n}E\left[e^{\kappa|\M_{ij}(n+1)|}\Big|\mathcal{F}_n\right] \leq C.
\end{equation}

    \item $\{a(n)\}$ is a non-negative sequence of stepsizes satisfying, for some $c>1$,
    \begin{equation*}
        \sum_{n}a(n)=\infty,\; \sum_{n}a^2(n) < \infty
    \end{equation*}
    \begin{equation*}
        a(m)\leq ca(n)\;\forall\; m\geq n.
    \end{equation*}
Furthermore, $\{a(n)\}$ is eventually decreasing and satisfies: If for $T > 0$, $m(n) := \min\{k \geq n : \sum_{i=n}^ka(i) \geq T\}$, then for a suitable constant $C^*$,
\begin{equation}\label{stepbdd}
m(n) - n \leq C^*T. 
\end{equation}This condition is satisfied by $a(n) = \frac{1}{1+n}, \frac{1}{1 + nlog n}$, etc., but not, e.g., by $a(n) = \frac{1}{1 + n^{2/3}}$.
\end{itemize}

We make the following key assumption:\\

\noindent \textbf{(A1)} Assume that the iterates remain a.s.\ bounded, i.e.,
\begin{equation}
\sup_n\|X(n)\| < \infty \ \mbox{a.s.}.
\label{stable}
\end{equation}

\medskip

Suppose the ODE 
\begin{equation}\label{diff-eq}
    \dot{\tilde{x}}(s) = \sum_{j=1}^{M}\pi(j)h^j(\tilde{x}(s))
\end{equation}
has an asymptotically stable  compact attractor $A$. 
Consider a bounded open subset $B' \subset \Re^{d}$ in the domain of attraction of $A$, positively invariant under (\ref{diff-eq}). In particular, all trajectories of (\ref{diff-eq}) initiated in $B'$ converge to $A$.  We also assume that a bounded open set $\breve{B}$ containing  $\tilde{B} := \{x: \inf_{y\in B'}\|x-y\| \leq \delta\}$ shares these properties for some $\delta > 0$ that we choose later.  Our objective will be to show that with high probability, the iterates $x^i(n), n \geq n_0$, remain in a small tube around  trajectory segments of (\ref{diff-eq}) initialized in $\breve{B}$ if it is sufficiently close to it at $n = n_0$ for some $n_0 \geq 0, \ \forall i$. 

We conclude this section with the remark that we can take the argument of $h^i(\cdot)$ above to be $X(n)$ instead of $x^i(n)$ by letting it be a map $\Re^{M\times d}\mapsto \Re^d$. This does not affect the analysis, but we lose the `distributed computation' viewpoint.

\section{Preliminaries}\label{prelim}

 Let $B = (B')^M, \hat{B} := (\breve{B})^M$. Suppose that $\|X_{n_0} - \Pi X_{n_0}\| < \delta$ and $\Pi X_{n_0} \in B $ for some $n_0\geq 0$. That these hold simultaneously  is possible with probability increasing to $1$ as $n_0\uparrow\infty$ because $\|X_{n_0} - \Pi X_{n_0}\|$ $\to 0$ a.s.\ \cite{Mathkar}. Fix a $T'>0$. Define $t(n) := \sum_{m=0}^na(m), n \geq 0$, and, for $\ k \geq 1$,
$$n_k:=\min\{n:t(n)\geq t(n_{k-1})+T'\}, \ \upsilon_k := n_{k+1} - n_k.$$
Define a sequence of time instants $T_0, T_1, T_2,\dots$ by $T_m = t(n_m)$ and let $I_m := [T_m,T_{m+1}]$. Then $T_{m+1}-T_m \leq T := T' + ca(0)$. Next, define a piecewise linear, continuous function $\bar{X}(t)$, $t\geq0$ such that $\bar{X}(t(n))=X(n), n\geq0,$ and it is linearly interpolated in each $I_n$.
Let $x^{T_m}(s) \in \Re^d , s \in I_m$, be the unique solution to the ODE (\ref{diff-eq}) with initial condition $x^{T_m}(T_m)= \sum_{j=1}^{M}\pi(j)x^j(T_m)$. Let $\Phi_{t}$ be the time-t flow map of the ODE \eqref{diff-eq} and let $C_{T}:=\max_{t\in [0,T],x\in\bar{B}}\|h(\Phi_{t}(x))\|_H<\infty$. Define
\begin{eqnarray}
  \rho_k &=& \sup_{t\in I_k} \|\bar{X}(t)-\textbf{1}(x^{T_k})^T(t)\|_H, \label{def-rho}\\
    \delta_{n,m}&:=& \sum_{i=1}^{m-1}a(n+i)P^{m-1-i}\widetilde{\M}(t(n+i+1)). \nonumber
\end{eqnarray}

\begin{lemma}\label{bound-lemma}
There exists a sequence $K_{T,n}^* > 0$ and a constant $K_T> 0$  such that $K_{T,n}^*\downarrow 0$ as $n\uparrow\infty$ and for any $k\geq0$,
\begin{equation}\label{bound}
    \rho_k \leq K_{T,k}^{*} + K_T\max_{\upsilon_k> l\geq 0}\delta_{n_k,l}.
\end{equation}
\end{lemma}
\begin{proof}
For any $n_k$,
$$
\bar{X}(t(n_k+1)) = P\bar{X}(t(n_k)) + a(n)\big(h(\bar{X}(t(n_k)))+\widetilde{\M}(n_k+1)\big).
$$
Iterating this through $n_k\leq n_k+m \leq n_{k+1}$,
\begin{eqnarray}
    \lefteqn{\bar{X}(t(n_k+m)) = P^m\bar{X}(t(n_k)) \nonumber} \\
    &+&\sum_{i=1}^{m-1}a(n_k+i)P^{m-1-i}h(\bar{X}(t(n_k+i)))\nonumber \\ &+&\sum_{i=1}^{m-1}a(n_k+i)P^{m-1-i}\widetilde{\M}(n_k+1). \label{one}
\end{eqnarray}
Defining $X^{T_m}(s) = \mathbf{1}(x^{T_m})^T(s)$,
\begin{eqnarray}
\lefteqn{X^{T_k}(t(n_k+m)) = \Pi X^{T_k}(t(n_k)) \nonumber} \\ &+& \int_{t(n_k)}^{t(n_k+m)}\Pi h(X^{T_k}(s))\, ds. \label{two}
\end{eqnarray}
Let  $\lfloor s\rfloor:= \max\{t(n):t(n)\leq s\}$. Subtracting  (\ref{two}) from (\ref{one}) and using $Q^k = P^k - \Pi$,
we get
\begin{align}
&\bar{X}(t(n_k+m)) - X^{T_k}(t(n_k+m)) = \left(P^{m}-\Pi\right)\bar{X}(t(n_k))\nonumber\\
&+\sum_{i=1}^{m-1}a(n_k+i)P^{m-1-i}h(\bar{X}(t(n_k+i)))\nonumber\\
&-\int_{t(n_k*)}^{t(n_k+m)}\Pi h(X^{T_k}(s))\; ds +\delta_{n_k,m} \nonumber\\
& =  Q^m\bar{X}(t(n_k))  + \sum_{i=0}^{m-1}a(n_k+i)\Pi \big(h(\bar{X}(t(n_k+i)))\nonumber\\
& \ \ \ \ \ \ \ \ \ - \ h(X^{T_k}(t(n_k+i)))\big)\nonumber\\
&+\sum_{i=1}^{m-1}a(n_k+i)Q^{m-1-i}h(\bar{X}(t(n_k+i)))\nonumber\\
&-\int_{t(n_k)}^{t(n_k+m)}\Pi \Big(h(X^{T_k}(s))-h(X^{T_k}(\lfloor s\rfloor))\Big)\; ds
+\delta_{n_k,m}. \label{newno}
\end{align}
Define 
\begin{equation}\label{definiton-z_m}
    z_m := \|\bar{X}(t(n_k+m)) - X^{T_k}(t(n_k+m))\|_H.
\end{equation}
 Taking $\| \cdot \|_H$ norm on both sides of (\ref{newno})and using Lemma \ref{contraction},
\begin{align}\label{term-2}
&z_m \leq \|Q^m\bar{X}(t(n_k))\|_H + \sum_{i=1}^{m-1}a(n_k+i)\|\Pi\|_HLz_i\nonumber\\
&+\sum_{i=1}^{m-1}a(n_k+i)\alpha^{m-1-i}\|h(\bar{X}(t(n_k+i)))\|_H\nonumber\\
&+\int_{t(n_k)}^{t(n_k+m)}\|\Pi \left(h(X^{T_k}(s))-h(X^{T_k}(\lfloor s\rfloor))\right)\|_Hds \nonumber\\
&+ \|\delta_{n_k,m}\|_H.
\end{align}
For the first term on the r.h.s.,
\begin{eqnarray}\label{term-1}
    \|Q^m\bar{X}(t(n_k))\|_H &=& \|Q^m\Big(\bar{X}(t(n_k))-\Pi\bar{X}(t(n_k))\Big)\|_H\nonumber\\
    &=&\|Q^m\Big(\bar{X}(t(n_k))-X^{T_k}(t(n_k))\Big)\|_H\nonumber\\
    &\leq&\alpha^mz_0.
\end{eqnarray}
For the third term,
\begin{align}\label{term-3}
&\sum_{i=1}^{m-1}a(n_k+i)\alpha^{m-1-i}\|h(\bar{X}(t(n_k+i)))\|_H\nonumber\\
&\leq\sum_{i=1}^{m-1}a(n_k+i)\alpha^{m-1-i}\|h(X^{T_k}(t(n_k+i)))\|_H \nonumber\\
&+\sum_{i=1}^{m-1}a(n_k+i)\alpha^{m-1-i}\nonumber\times\\
&\|h(\bar{X}(t(n_k+i)))-h(X^{T_k}(t(n_k+i)))\|_H\nonumber\\
&\leq C_T\sum_{i=0}^{m-1}a(n+i)\alpha^{m-1-i} + L\sum_{i=0}^{m-1}a(n_k+i)\alpha^{m-1-i}z_i.
\end{align} 
By the Cauchy-Schwarz inequality,
\begin{eqnarray*}
\sum_{i=1}^{m-1}a(n+i)\alpha^{m-1-i} &\leq& \sqrt{\sum_{i=1}^{m-1}a^2(n_k+i)\sum_{i=1}^{m-1}\alpha^{2i}}\\
&\leq&\sqrt{\frac{\sum_{i\geq 1}a^2(n_k+i)}{1-\alpha^2}}.
\end{eqnarray*}
Note that if $t(n_k+m)\leq t\leq t(n_k+m+1)$, then
\begin{align*}
    &\|X^{T_m}(t)-X^{T_m}(t(n_k+m))\|_H \\
&\leq \left\|\int_{t(n_k+m)}^t h(X^{T_m}(s)) ds\right\|_H\\
    &\leq C_T(t-t(n_k+m))\\
    &\leq C_Ta(n_k+m).
\end{align*}
Now  consider the fourth term.
\begin{align}\label{term-4}
&\int_{t(n_k)}^{t(n_k+m)}\|\Pi \left(h(X^{T_k}(s))-h(X^{T_k}(\lfloor s\rfloor))\right)\|_H ds  \nonumber\\
&\leq\int_{t(n_k)}^{t(n_k+m)}L\|\Pi\|_H \|X^{T_k}(s)-X^{T_k}(\lfloor s\rfloor)\|_H ds\nonumber\\
&\leq L\|\Pi\|_H\sum_{i=1}^{m-1}\int_{t(n_k+i)}^{t(n_k+i+1)}\|X^{T_k}(s)-X^{T_k}(t(n_k+i))\|_H \nonumber\\
&\leq L\|\Pi\|_HC_T\sum_{i=1}^{m-1}a^2(n_k+i).
\end{align}
Combining \eqref{term-2}, \eqref{term-1}, \eqref{term-3} and \eqref{term-4},
\begin{equation}
 z_m \leq K_{T,m} + \sum_{i=1}^{m-1}q_iz_i.
\end{equation}
Here, for $b(n) := \sum_{m\geq n}a(n)^2$, 
\begin{align*}
K_{T,k} &=  LC_T\|\Pi\|_Hb(n_k) + C_T\sqrt{\frac{b(n_k)}{1-\alpha^2}} + \max_{\upsilon_k> l\geq 0}\|\delta_{n_k,l}\|_H,\\ 
q_i =& La(n+i)(1+\|\Pi\|_H) \ \forall \ i\neq0,\\
q_0 =& La(n)(1+\|\Pi\|_H)+1.
\end{align*} 
Also, $\sum_{i=0}^{m-1}q_i \leq LT(1+\|\Pi\|_H)+1$.  Using discrete  Gronwall inequality,
\begin{eqnarray*}
    z_m &\leq& K_{T,n_k}\exp\left(\sum_{i=0}^{m-1}q_i\right) \\
    \implies z_m &\leq& K_{T,n_k}\exp(LT(1+\|\Pi\|_H+1)).
\end{eqnarray*}
For $t(n_k+m)\leq t \leq t(n_k+m+1)$, we have
$$
\bar{X}(t) = \lambda\bar{X}(t(n_k+m))+(1-\lambda)\bar{X}t(n_k+m+1)
$$
 for some $\lambda \in [0,1]$. Then a routine calculation as in \cite{Borkar}, pp.\ 15-16, shows that 
\begin{equation}
 \rho_k :=   \sup_{t\in I_k} \|\bar{X}(t)-X^{T_k}(t)\|_H \leq K_{T,k}^{*} + K_T\max_{\upsilon_k> l\geq 0}\delta_{n_k,l}
\end{equation}
for suitably defined $K_{T,k}^{*}$ and $K_T$, where $K_{T,k}^{*}\to 0$ as $n_k \uparrow \infty$.
\end{proof}

Let $V: \hat{B} \rightarrow [0,\infty)$ be a continuously differentiable  function such that $\lim_{x\to\partial \hat{B}}V(x) = \infty$, $\dot{V} := \langle \nabla V,h \rangle: \hat{B}\rightarrow \mathbf{R}$ is non-positive,  $A:= \{x: V(x)=0\}  = \{x: \dot{V}(x)=0\}.$ That is, V is a local Lyapunov function associated with the attractor $A$  of  \eqref{diff-eq}. ($V$ exists by the converse Lyapunov theorem \cite{Krasovskii}.)

Fix  $\epsilon \geq 0$ such that the set $\overline{A^{\epsilon}}:= \{x: V(x)\leq\epsilon\} \subset B$. Fix some $T> 0$ and let $\Delta>0$ satisfy
\begin{equation}\label{delta}
    \Delta < \min_{x\in \hat{B}\backslash A^\epsilon} [V(x)-V(\Phi_T(x))]. 
\end{equation}
For any set $D$ and $\delta > 0$, denote the $\delta$-neighbourhood of $D$ by $N^{(\delta)}(D)$. Next, pick  $\delta \geq 0 $ in the definition of $\hat{B}$ such that $N^{(\delta)}(A^\epsilon) \subset B$ and furthermore, if $x \in B$ and $\|x-y\|_H < \delta$, then $|V(x)-V(y)| < \frac{\Delta}{3}$. By increasing $T$ if necessary, assume that for any trajectory $x(\cdot)$ of (\ref{diff-eq}) initiated in $\bar{B'}$, $x(T)$ is at least $\delta$ away from $\partial B'$.
Now we state our main result. Define
$$\tilde{\delta} := \frac{\delta}{2K_T\sqrt{\Lambda(H)M^3d}}.$$

\begin{theorem}
There exist constants $C_1,C_2,\tau>0$ such that if $n_0$ is `sufficiently large', then
\begin{align*}
    &P\Big(\bar{X}(t)\in N^{(\delta)}\left(A^{\epsilon+\frac{2\Delta}{3}}\right) \;\forall\; t\geq T_0+\tau \big| \\
& \ \ \ \ \ \ \ \ \ \ \ \ \ \ \ \ \ \|X_{n_0}-\Pi X_{n_0}\| < \delta,  \Pi X_{n_0} \in B \Big)  \nonumber\\
    &\geq 1- 2M^2dC^*\sum_{n\geq n_0}ne^{-\frac{D\tilde{\delta}^2}{a(n)}} \ \ \mbox{if} \ \tilde{\delta} \in \Big(0,\frac{CT}{\varepsilon}\Big], \\
& \ \ \ \ 1 -  2M^2dC^*\sum_{n\geq n_0}ne^{-\frac{D\tilde{\delta}}{a(n)}} \ \ \mbox{otherwise}. 
\end{align*}
\end{theorem}

We prove this in the next section.

\section{Proof of the Main result}

Suppose $\rho_{k} < \delta \; \forall \; k \geq 0$. If $\Pi\bar{X}(T_{0}) \in B\backslash A^{\epsilon}$, then $X^{T_0}(t) \in B$ for $t\in[T_0,T_1]$ and because $\Pi\bar{X}(T_{0}) = X^{T_0}(T_0)$ and \eqref{delta}, we have 
$$
V(X^{T_0}(T_1)) \leq V(X^{T_0}(T_0)) - \Delta.
$$
Since $\|\bar{X}(t)-X^{T_0}(t)\|_H < \delta, t \in [T_0,T_1]$, we have
$$
V(\bar{X}(T_1)) \leq V(\bar{X}(T_0)) - \frac{\Delta}{3}
$$
and $\bar{X}(t) \in \hat{B}, t \in [T_0,T_1]$. Also, since $\Pi Q = Q\Pi =$ the zero matrix, it follows from (\ref{Lia}) that $\|\Pi x\|_H = \|\Pi x\|$. Clearly $\Pi X^{T_0}(t) = X^{T_0}(t)$. Thus we also have $\|\Pi \bar{X}(t) - X^{T_0}(t)\|_H < \delta$. If $\Pi\bar{X}(T_1) \in \hat{B}\backslash A^{\epsilon}$ then, the same argument can be repeated. But since $V(\bar{X}(T_m))$ cannot decrease indefinitely, there exists some $m_0$ such that $\Pi \bar{X}(T_{m_0}) \in A^{\epsilon}$. In fact,
$$\tau:= \frac{3(\max_{x\in \bar{B}} V(x) - \epsilon)}{\Delta}(T+1) \Longrightarrow T_{m_0} \leq T_0 + \tau.$$
Thus $X^{T_{m_0}}(T_{m_0}) = \Pi \bar{X}(T_{m_0}) \in A^{\epsilon}$, implying $X^{T_{m_0}}(t) \in A^{\epsilon}$ for $t \in I_{m_0}$. Therefore $\bar{X}(T_{m_0+1}) \in A^{\epsilon + \frac{\Delta}{3}}$. Again, because $\|\bar{X}(T_{m_{0}+1})-X^{T_{m_{0}+1}}(T_{m_{0}+1})\|_H < \delta$, we have $X^{T_{m_{0}+1}}(T_{m_{0}+1})=\Pi \bar{X}(T_{m_{0}+1})\in A^{\epsilon+\frac{2\Delta}{3}}$. Now there are two possibilities: either $X^{T_{m_{0}+1}}(T_{m_{0}+1}) \in A^{\epsilon}$ or $X^{T_{m_{0}+1}}(T_{m_{0}+1}) \in A^{\epsilon+\frac{2\Delta}{3}}\backslash A^{\epsilon}$. In the former case, $X^{T_{m_{0}+1}}(t) \in A^{\epsilon}$ for $t \in I_{{m_0}+1}$, implying $X^{T_{m_{0}+1}}(T_{m_0+2}) \in A^{\epsilon}$. Since $\|\bar{X}(T_{m_{0}+2})-X^{T_{m_{0}+1}}(T_{m_{0}+2})\|_H < \delta$,  $\bar{X}(T_{m_{0}+2}) \in A^{\epsilon+\frac{\Delta}{3}}$. Otherwise, if $X^{T_{m_{0}+1}}(T_{m_{0}+1}) \in A^{\epsilon+\frac{2\Delta}{3}}\backslash A^{\epsilon}$, $X^{T_{m_{0}+1}}(t) \in A^{\epsilon + \frac{\Delta}{3}}$ on $I_{{m_0}+1}$ and there will be a drop of $\Delta$ in $V$ along the trajectory of ODE from $T_{m_{0}+1}$ to $T_{m_{0}+2}$. Therefore $X^{T_{m_{0}+1}}(T_{m_{0}+2}) \in A^{\epsilon - \frac{\Delta}{3}} \subset A^{\epsilon}$. This again implies that $\bar{X}(T_{m_{0}+2}) \in A^{\epsilon + \frac{\Delta}{3}}$. This argument can be repeated. Thus we can conclude that if for some $n_0$, $\Pi X(n_0) \in B$ and $\rho_k < \delta \; \forall \; k \geq 0$, then there exists some $m_0$ such that $X^{T_{m_{0}+k}}(t) \in A^{\epsilon+\frac{2\Delta}{3}}$ on $I_{m_0+k}$ for all $k\geq0$, implying that $\bar{X}(t) \in N^{(\delta)}(A^{\epsilon+\frac{2\Delta}{3}})$ for $t\geq T_{m_0}$, hence for $t\geq T_{0}+\tau$. 
Let $\mathcal{B}_k := \mathcal{B}_{-1}\cap\{\rho_m < \delta \ \forall \ 0 \leq m \leq k\}$, with $\mathcal{B}_{-1}:=  \{ \| X_{n_0} - \Pi X_{n_0} \| < \delta,  \ \Pi X_{n_0} \in B \}$.
Then by the foregoing,
\begin{align}\label{main-result}
    &P\left(\bar{X}(t)\in N^{(\delta)}(A^{\epsilon+\frac{2\Delta}{3}}) \ \forall\; t\geq T_0+\tau\Big|\mathcal{B}_{-1} \right)  \nonumber\\
    &\geq P\big(\rho_k<\delta \ \forall\; k\geq0| \mathcal{B}_{-1}\big)
\end{align}
 Thus we have:
\begin{lemma}
\begin{align}\label{prob-summation}
     &P\big(\rho_k<\delta \ \forall\; k\geq0|\mathcal{B}_{-1}\big) \nonumber\\
     &\geq 1- \sum_{k=0}^{\infty}P\big(\rho_k\geq\delta|\mathcal{B}_{k-1}\big).
\end{align}
\end{lemma}
This follows as in \cite{Borkar}, p.\ 25. 
Recall that $\upsilon_k := n_{k+1}-n_k$. If we select $n_0$ large enough so that 
\begin{equation}\label{condition-1}
 K_{T,k}<\frac{\delta}{2} \ \  \forall\;k\geq0,   
\end{equation}
then the inequality $\rho_{k}\geq\delta$ and Lemma \ref{bound-lemma} together imply that $\max_{\upsilon_k > l \geq0}\|\delta_{n_k,l}\|_H>\frac{\delta}{2K_T}$. Thus
$$
P\left(\rho_k\geq\delta|\mathcal{B}_{k-1}\right)\leq P\left(\max_{\upsilon_k> l\geq0}\|\delta_{n_k,l}\|_H>\frac{\delta}{2K_T}\Big|\mathcal{B}_{k-1}\right)
$$

The next  lemma establishes a useful bound.
\begin{lemma}
There exists a constant $K_3$ such that for any $t\in I_k$,
\begin{equation*}
\|\bar{X}(t)\|_2\leq K_3\left(1+\|\bar{X}(T_k)\|_2
\right).
\end{equation*}
\end{lemma}
\begin{proof}
We have
$$
X(n_k+1) = PX(n_k) + a(n_k)\big(h(X(n_k))+\widetilde{\M}(n_k+1)\big).
$$
Taking $\|\cdot\|_2$ norm and using the linear growth property of the function $h$ and of the martingale difference noise,
\begin{align*}
    &\|X(n_k+1)\|_2 \leq \|P\|_2\|X(n_k)\|_2 \\
    &+ a(n_k)\big(K_1(1+\|X(n_k)\|_2)+K_2(1+\|X(n_k)\|_2)\big).
\end{align*}
But because $P$ is a stochastic matrix, $\|P\|_2\leq 1$, so
\begin{align*}
    \|X(n_k+1)\|_2 \leq \|X(n_k)\|_2(1+K^{'}a(n_k)) +a(n_k)K^{'},
\end{align*}
where $K^{'}:=K_1+K_2$.
Iterating this for $n_k\leq j\leq n_{m+1}$ and using the fact that $1+aK^{'}\leq \exp(aK')$ and the fact that if $j<n_{k+1}$, then $a(n_k)+a(n_k+1)+\dots+a(j-1)<t(n_{k+1})-t(n_{k})<T$, we get
\begin{align*}
    \|X(j)\|_2 \leq& \|X(n_k)\|_2\exp(K^{'}T) \ + \\
& \ \ K'T\exp(K^{'}T)\\
    \implies& \ \|X(j)\|_2 \leq K_3(1+\|X(n_k)\|_2)
\end{align*}
for some constant $K_3>0$.
\end{proof}

 Let $[\cdots]_{a,b}$ denote the $(a,b)$th term of the matrix `$\cdots$' and let $\Lambda(H) := \lambda_{max}(H)/\lambda_{min}(H)$. Then for any $G\in \Re^{M\times d}$, 
$$\|G\|_H\leq \sqrt{\Lambda(H)}\|G\|_F\leq \sqrt{\Lambda(H)Md}\max_{i,j}[G]$$ 
where $\| \cdot \|_F$ is the Frobenius norm. Thus we have
\begin{align}\label{probability-ref}
    &P\left(\max_{\upsilon_k> l\geq0}\|\delta_{n_k,l}\|_H>\frac{\delta}{2K_T}\Big|\mathcal{B}_{k-1}\right)\nonumber\\
    &\leq P\left(\max_{\upsilon_k> l\geq0} \max_{a,b}|[\delta_{n_k,l}]_{a,b}|>\frac{\delta}{2K_T\sqrt{\Lambda(H)Md}}\Big|\mathcal{B}_{k-1}\right)\nonumber\\
    &\leq \sum_{a,b} P\left(\max_{\upsilon_k> l\geq0}|[\delta_{n_k,l}]_{a,b}|>\frac{\delta}{2K_T\sqrt{\Lambda(H)Md}}\Big|\mathcal{B}_{k-1}\right).
\end{align}
But 
\begin{eqnarray*}
 \lefteqn{   |[\delta_{n_k,l}]_{a,b}| } \\
&=& \left|\sum_{i=0}^{l-1}a(n_k+i)\sum_{c=1}^{M}[P^{l-1-i}]_{a,c}[\M]_{c,b}(n_k+i+1)\right|\\
    &\leq& \sum_{c=1}^{M}\left|\sum_{i=0}^{l-1}a(n_k+i)[P^{l-1-i}]_{a,c}[\M]_{c,b}(n_k+i+1)\right|\\
    &\leq& M\max_{1\leq c\leq M}\Bigg|\sum_{i=0}^{l-1}a(n_k+i)[P^{l-1-i}]_{a,c}\times \\
&& \ \ \ \ \ \ \ \ \ \ [\M]_{c,b}(n_k+i+1)\Bigg|.
\end{eqnarray*}
 
Combining this with \eqref{probability-ref},
\begin{eqnarray*}
\lefteqn{P\left(\max_{\upsilon_k> l\geq0}\|\delta_{n_k,l}\|_H>\frac{\delta}{2K_T} \ \Big| \ \mathcal{B}_{k-1}\right) \ \leq} \\
&& \sum_{a,b,c} \sum_{i=0}^{\upsilon_k-1}P\Bigg(\Big|\sum_{i=0}^{l-1}a(n_k+i)[P^{l-1-i}]_{a,c}\M_{c,b}(n_k+i+1)\Big|\\
&& \ \ \ \ \ \ \ \ \  > \ \frac{\delta}{2K_T\sqrt{\Lambda(H)M^3d}} \ \Big| \ \mathcal{B}_{k-1}\Bigg).
\end{eqnarray*}
Let $S_m:=\sum_{i=0}^{m-1}a(n_k+i)[P^{l-1-i}]_{a,c}[\M]_{c,b}(n_k+i+1)$. Then $\{(S_m, \mathcal{F}_{n_k+m}), 0 \leq m \leq l\}$ is a martingale  and $\mathcal{B}_{n_k+m}\in \mathcal{F}_{n_k+m}.$ Let $Y_1 = S_1$ and $Y_m=S_m-S_{m-1}\;\forall\; l\geq m>1$. Then
\begin{align*}
    |Y_i| &= |a(n_k+i)[P^{l-1-i}]_{a,c}[\M]_{c,b}(n_k+i+1)| \\
    &\leq a(n_k+i)|[\M]_{c,b}(n_k+i+1)| \\
&\leq a(n_k+i)\|\M(n_k+i+1)\|_{\infty}\\
    &\leq da(n_k+i)\|\M(n_k+i+1)\|_2\\
    &\leq da(n_k+i)K_2(1+\|X(n_k+i+1)\|_2)\\
    &\leq da(n_k+i)K_2(1+K_3(1+\|\bar{X}(T_k)\|_2)) \\
    &\leq K_4da(n_k+i),
\end{align*}
 where $K_4:= K_2(1+K_3(1+K_5))$ for $K_5 = sup_{X\in \hat{B}}\|X\|_2$.
 
Now we use the conditional version of the inequality for martingale arrays from the Appendix with $A_{m,l} = a(m)p_{a,c}^{l - m - 1}, \gamma_1 = T$ and $\gamma_2 =c, w(m) = a(n_k)$ for $0\leq m \leq l$. Recall that
$$\tilde{\delta} := \frac{\delta}{2K_T\sqrt{\Lambda(H)M^3d}}.$$
Then for a constant $D > 0$ depending on $\kappa, C, T, c$, 
\begin{eqnarray*}
 &&\sum_{a,b,c,l}P\Bigg(\left|\sum_{i=0}^{l-1}a(n_k+i)[P^{l-1-i}]_{a,c}[\M]_{c,b}(n_k+i+1) \ \right| \\
&& \ \ \ \ \ \ \ \ \ \ \ \ \ \ \ >  \tilde{\delta} \ \Big| \ \mathcal{B}_{k-1}\Bigg)\\
&&\leq 2M^2d(n_{k+1}-n_k)e^{-\frac{D\tilde{\delta}^2}{a(n_k)}} \ \ \mbox{if} \ \tilde{\delta} \in \Big(0,\frac{CT}{\kappa}\Big], \\
&& \ \ \  2M^2d(n_{k+1}-n_k)e^{-\frac{D\tilde{\delta}}{a(n_k)}} \ \ \mbox{otherwise} \\
&&\leq 2M^2dC^*n_ke^{-\frac{D\tilde{\delta}^2}{a(n_k)}} \ \ \mbox{if} \ \tilde{\delta} \in \Big(0,\frac{CT}{\kappa}\Big], \\
&& \ \ \  2M^2dC^*n_ke^{-\frac{D\tilde{\delta}}{a(n_k)}} \ \ \mbox{otherwise}, 
\end{eqnarray*}
where the summations over $a,b,c,l$ are over their respective ranges and $C^*$ is as in (\ref{stepbdd}).

Combining this with \eqref{main-result} and \eqref{prob-summation} we get,
\begin{align}\label{pre-main}
    &P\Big(\bar{X}(t)\in N^{(\delta)}\left(A^{\epsilon+\frac{2\Delta}{3}}\right) \;\forall\; t\geq T_0+\tau\Big|\nonumber \\
& \ \ \ \ \ \ \ \ \ \ \ \ \|X_{n_0}  - \Pi X_{n_0}\| < \delta, \Pi X_{n_0}\in B \Big)  \nonumber\\
    &\geq 1- 2M^2dC^*\sum_{n\geq n_0}ne^{-\frac{D\tilde{\delta}^2}{a(n)}} \ \ \mbox{if} \ \tilde{\delta} \in \Big(0,\frac{CT}{\kappa}\Big] \\
& \ \ \ \  1 -  2M^2dC^*\sum_{n\geq n_0}ne^{-\frac{D\tilde{\delta}}{a(n)}} \ \ \mbox{otherwise}. 
\end{align}
This completes the proof of the main result.

\section{Future directions}

This is only a first step towards a  finer analysis of this class of algorithms. In fact there have been other variants of the basic scheme (\ref{single-eq}), one of the important strands being the case when instead of a weighted average of the iterate over all neighbors, each processor polls one of the neighbors with a prescribed probability and takes an average of its own iterate with that of the selected neighbor, see, e.g., \cite{Nedic}. The development above does not carry over to this case automatically and will have to now include analysis of a product of random stochastic matrices.

A further extension of the basic paradigm itself has been to replace the averaging operator by a nonlinear map so that asymptotically, the ODE limit of the stochastic approximation gets confined to the set of its fixed points \cite{Mathkar}. This idea has been used for distributed projected stochastic approximation where the projection is also enforced by a distributed scheme \cite{Shah}. It will be challenging and useful to extend the foregoing analysis to cover such scenarios.

Finally, there is always scope to tighten the bounds above, as we have been rather generous with them at various points.\\

\noindent \textbf{APPENDIX}\\

Let $\{\M_n\}$ be a real valued martingale difference sequence with respect to an increasing family of $\sigma$-fields $\{\mathcal{F}_n\}$. Assume that there exist $\varepsilon, C > 0$ such that
$$E\left[e^{\varepsilon |\M_n|}\Big|\mathcal{F}_{n-1}\right] \leq C \ \ \forall \; n \geq 1, \mbox{a.s.}$$
Let $S_n := \sum_{m=1}^n\xi_{m,n}\M_m$, where $\xi_{m,n}, \ m \leq n,$, for each $n$, are a.s.\ bounded $\{\mathcal{F}_n\}$-previsible random variables, i.e., $\xi_{m,n}$ is $\mathcal{F}_{m-1}$-measurable $\forall \; m \geq 1$, and $|\xi_{m,n}| \leq A_{m,n}$ a.s.\ for some constant $A_{m,n}$, $\forall \; m, n$. Suppose
$$\sum_{m=1}^nA_{m,n} \leq \gamma_1, \ \max_{1\leq m \leq n}A_{m,n} \leq \gamma_2\omega(n),$$
for some $\gamma_i, \omega(n) > 0, \ i = 1,2; n \geq 1$. Then we have:

\begin{theorem}\label{thm-appendix} There exists a constant $D > 0$ depending on $\varepsilon, C, \gamma_1, \gamma_2$ such that for $\epsilon > 0$,
\begin{eqnarray}
P\left(|S_n| > \epsilon\right) &\leq& 2e^{-\frac{D\epsilon^2}{\omega(n)}}, \ \ \mbox{if} \ \epsilon \in \left(0, \frac{C\gamma_1}{\varepsilon}\right], \label{LW1} \\
&&  2e^{-\frac{D\epsilon}{\omega(n)}},  \ \ \mbox{otherwise.} \label{LW2}
\end{eqnarray}
\end{theorem}

This is a variant of Theorem 1.1 of \cite{Liu}. See \cite{Gugan}, Theorem A.1, pp.\ 21-23, for details.

\end{document}